\documentclass[
a4paper,orivec]{llncs}

\usepackage{amssymb}
\usepackage{graphicx}

\usepackage{url}
\urldef{\mailsa}\path|jose.m.pena@liu.se|

\usepackage{graphicx,amssymb,datetime,float,MnSymbol,currfile,multirow}
\usepackage{tikz}
\usetikzlibrary{arrows}

\def\ci{\!\perp\!}

\def\ra{\rightarrow}
\def\la{\leftarrow}

\newcommand{\comments}[1]{}

\begin{document}

\mainmatter  

\title{Every LWF and AMP Chain Graph Originates from a Set of Causal Models}

\titlerunning{Every LWF and AMP Chain Graph Originates from a Set of Causal Models}

%
%
\author{Jose M. Pe\~{n}a}
\authorrunning{Every LWF and AMP Chain Graph Originates from a Set of Causal Models}

\institute{ADIT, IDA, Link\"oping University, SE-58183 Link\"{o}ping, Sweden\\
\mailsa}

%
%

\toctitle{Lecture Notes in Computer Science}
\tocauthor{Authors' Instructions}
\maketitle

\begin{abstract}
This paper aims at justifying LWF and AMP chain graphs by showing that they do not represent arbitrary independence models. Specifically, we show that every chain graph is inclusion optimal wrt the intersection of the independence models represented by a set of directed and acyclic graphs under conditioning. This implies that the independence model represented by the chain graph can be accounted for by a set of causal models that are subject to selection bias, which in turn can be accounted for by a system that switches between different regimes or configurations.
\end{abstract}

\section{Introduction}

Chain graphs (CGs) are graphs with possibly directed and undirected edges, and no semidirected cycle. They have been extensively studied as a formalism to represent independence models. CGs extend Bayesian networks (BNs), i.e. directed and acyclic graphs (DAGs), and Markov networks, i.e. undirected graphs. Therefore, they can model symmetric and asymmetric relationships between the random variables of interest. This was actually one of the main reasons for developing them. However, unlike Bayesian and Markov networks whose interpretation is unique, there are three main interpretations of CGs as independence models: The Lauritzen-Wermuth-Frydenberg (LWF) interpretation \cite{Frydenberg1990,LauritzenandWermuth1989}, the multivariate regression (MVR) interpretation \cite{CoxandWermuth1993,CoxandWermuth1996}, and the Andersson-Madigan-Perlman (AMP) interpretation \cite{Anderssonetal.2001,Levitzetal.2001}. A fourth interpretation has been proposed in \cite{Drton2009} but it has not been studied sufficiently and, thus, it will not be discussed in this paper. It should be mentioned that any of the three main interpretations can represent independence models that cannot be represented by the other two interpretations \cite{SonntagandPenna2013}.

Along with other reasons, DAGs can convincingly be justified by the fact that each of them represents a causal model. Whether this is an ontological model is still debated. However, it is widely accepted that the causal model is at least epistemological and thus worth studying \cite{Pearl2000}. Of the three main interpretations of CGs, however, only MVR CGs have a convincing justification: Since MVR CGs are a subset of maximal ancestral graphs without undirected edges, every MVR CG represents the independence model represented by a DAG under marginalization \cite[Theorem 6.4]{RichardsonandSpirtes2002}. That is, every MVR CG can be accounted for by a causal model that is partially observed. Unfortunately, LWF and AMP CGs cannot be justified in the same manner because (i) LWF and AMP CGs can represent independence models that cannot be represented by maximal ancestral graphs \cite[Section 9.4]{RichardsonandSpirtes2002}, and (ii) maximal ancestral graphs can represent all the independence models represented by DAGs under marginalization and conditioning \cite[Theorem 4.18]{RichardsonandSpirtes2002}. In other words, LWF and AMP CGs can represent independence models that cannot be represented by any DAG under marginalization and conditioning. Of course, LWF and AMP CGs can be justified by the fact that they improve the expressivity of DAGs, i.e. they can represent more independence models than DAGs \cite{Penna2007}. However, this is a weak justification unless those independence models are not arbitrary but induced by some class of knowledge representatives within some uncertainty calculus of artificial intelligence, e.g. the class of probability distributions \cite[Section 1.1]{Studeny2005}. This is exactly what the authors of \cite{Levitzetal.2001,Penna2009,Penna2011a,StudenyandBouckaert1998} do by showing that every LWF and AMP CG is faithful to some probability distribution. However, this does not strengthen much the justification unless these probability distributions are not arbitrary but they represent meaningful systems or phenomena. This is exactly what the authors of \cite{LauritzenandRichardson2002} do. In particular, the authors show that every LWF CG includes the independence model induced by the equilibrium probability distribution of a dynamic model with feed-back. The downside of this justification is that the equilibrium distribution may not be reached in finite time and, thus, it may not coincide with the distribution that represents the behaviour of the dynamic model at any finite time point. Therefore, there is no guarantee that the CG includes the independence model induced by the latter, which is the goal. The authors are aware of this and state that their justification should better be understood as an approximated one. Another work in the same vein is \cite{Ferrandizetal.2005}, whose authors show that some LWF CGs are inclusion minimal wrt the result of temporal aggregation in a DAG representing a spatio-temporal process. Unfortunately, the authors do not show whether their result holds for every LWF CG. Yet another work along the same lines is \cite{Penna2014}, whose author shows that every AMP CG is faithful to the independence model represented by a DAG under marginalization and conditioning. It is worth noting that the DAG contains deterministic nodes, because the result does not hold otherwise \cite{Richardson1998}. Finally, the author of \cite{Studeny1998} presents the following justification of LWF CGs. Each connectivity component of a LWF CG models an area of expertise. The undirected edges in the connectivity component indicate lack of independencies in the area of expertise. The directed edges in the CG indicate which areas of expertise are prerequisite of which other areas. However, the author does not describe how the independencies in the local models of the areas of expertise get combined to produce a global model of the domain, and how this model relates to the one represented by the CG.

In this work, we show that every LWF and AMP CG $G$ is inclusion optimal wrt the intersection of the independence models represented by a set of DAGs under conditioning. In other words, we show that (i) the independencies represented by $G$ are a subset of the intersection, and (ii) the property (i) is not satisfied by any CG that represents a proper superset of the independencies represented by $G$. Note that if there exists a CG that is faithful to the intersection, then that CG is inclusion optimal. In general, several inclusion optimal CGs exist and they do not necessarily represent the same independence model. Therefore, in principle, one prefers the inclusion optimal CGs that represent the largest number of independencies. However, finding any such CG seems extremely difficult, probably NP-complete in the light of the results in \cite{Penna2011b}. Thus, one is typically content with finding any inclusion optimal CG. An example of this are the algorithms for learning inclusion optimal BNs \cite{ChickeringandMeek2002,Nielsenetal.2003} and LWF CGs \cite{Pennaetal.2014}. This is also why we are content with showing in this paper that every LWF and AMP CG is inclusion optimal wrt the intersection of the independence models represented by a set of DAGs under conditioning. The intersection can be thought of as a consensus independence model, in the sense that it contains all and only the independencies upon which all the DAGs under conditioning agree. We elaborate further on the term consensus in the paragraph below. The fact that every LWF and AMP CG originates from a set of DAGs under conditioning implies that the independence model represented by the former can be accounted for by a set of causal models that are subject to selection bias, which in turn can be accounted for by a system that switches between different regimes or configurations. Two examples of such a system are the progression of a disease through different stages, and the behaviour of a broker alternating between looking for buying and selling opportunities. We have recently introduced a new family of graphical models aiming at modeling such systems \cite{BendtsenandPenna2013,BendtsenandPenna2014}. In summary, we provide an alternative justification of LWF and AMP CGs that builds solely on causal models and does not involve equilibrium distributions or deterministic nodes, which may seem odd to some readers. Our hope is that this strengthens the case of LWF and AMP CGs as a useful representation of the independence models entailed by causal models.

Before we proceed further, it is worth discussing the relationship between our justification of LWF and AMP CGs and belief aggregation. First, recall that a BN is an efficient representation of a probability distribution. Specifically, a BN consists of structure and parameter values. The structure is a DAG representing an independence model. The parameter values specify the conditional probability distribution of each node given its parents in the BN structure. The BN represents the probability distribution that results from the product of these conditional probability distributions. Moreover, the probability distribution satisfies the independence model represented by the BN structure. Belief aggregation consists in obtaining a group consensus probability distribution from the probability distributions specified by the individual members of the group. Probably, the two most commonly used consensus functions are the weighted arithmetic and geometric averages. The authors of \cite{PennockandWellman2005} show that belief aggregation is problematic when the consensus and the individual probability distributions are represented as BNs. Specifically, they show that even if the group members agree on the BN structure, there is no sensible consensus function that always returns a probability distribution that can be represented as a BN whose structure is equivalent to the agreed one \cite[Proposition 2]{PennockandWellman2005}. The only exception to this negative result is when the individual BN structures are decomposable and the consensus function is the weighted geometric average \cite[Sections 3.3-3.4]{PennockandWellman2005}. However, the authors also point out that this negative result does not invalidate the arguments of those who advocate preserving the agreed independencies, e.g. \cite{Laddaga1977} and \cite[Section 8.12]{Raiffa1968}. It simply indicates that a different approach to belief aggregation is needed in this case. They actually mention one such approach that consists in performing the aggregation in two steps: First, find a consensus BN structure that preserves as many of the agreed independencies as possible and, second, find consensus parameter values for the consensus BN structure. The first step has received significant attention in the literature \cite{delSagradoandMoral2003,MatzkevichandAbramson1992,MatzkevichandAbramson1993a,MatzkevichandAbramson1993b,NielsenandParsons2007}. A work that studies both steps is \cite{Bonduelle1987}.\footnote{Unfortunately, we could not get access to this work. So, we trust the description of it made in \cite[Section 3.5]{PennockandWellman2005}.} We have also studied both steps \cite{Etminanietal.2013,Penna2011b}. The two step approach described above is also suitable when some of the group members are able to contribute with a BN structure but not with parameter values. This scenario is not unlikely given that people typically find easier to gather qualitative than quantitative knowledge.

Our justification of LWF and AMP CGs implicitly advocates preserving the agreed independencies, because the DAGs in the justification are combined through the intersection of the independence models that they represent and, thus, the agreed independencies are kept. As shown above, this is a sensible advocation. Therefore, in this paper we make use of it to propose a sensible justification of LWF and AMP CGs. The DAGs in our justification are hand-picked to ensure that the combination thereof produces the desired result. This raises the question of how to combine a set of arbitrary DAGs under marginalization and conditioning into a LWF or AMP CG. In this paper, we also investigate this question. Ideally, we would like to find a LWF or AMP CG that is inclusion optimal wrt the intersection of the independence models represented by the DAGs under marginalization and conditioning. Unfortunately, this problem seems extremely hard. So, we actually study a simpler version of it. Note that this problem corresponds to the first step of the approach to belief aggregation described above. The second step, i.e. combining the parameter values associated to the DAGs, is beyond the scope of this paper.

The rest of the paper is organized as follows. In Section \ref{sec:preliminaries}, we introduce some preliminaries and notation. In Section \ref{sec:justification}, we present our justification of LWF and AMP CGs. In Section \ref{sec:combination}, we discuss how to combine arbitrary DAGs into a LWF or AMP CG. We close with some discussion in Section \ref{sec:discussion}.

\section{Preliminaries}\label{sec:preliminaries}

In this section, we review some concepts from graphical models that are used later in this paper. Unless otherwise stated, all the graphs in this paper are defined over a finite set $V$. Moreover, they are all simple, i.e. they contain at most one edge between any pair of nodes. The elements of $V$ are not distinguished from singletons. The set operators union, intersection and difference are given equal precedence in the expressions. The term maximal is always wrt set inclusion.

If a graph $G$ contains an undirected or directed edge between two nodes $V_{1}$ and $V_{2}$, then we write that $V_{1} - V_{2}$ or $V_{1} \ra V_{2}$ is in $G$. The parents of a set of nodes $X$ of $G$ is the set $pa_G(X) = \{V_1 | V_1 \ra V_2$ is in $G$, $V_1 \notin X$ and $V_2 \in X \}$. The children of $X$ is the set $ch_G(X) = \{V_1 | V_1 \la V_2$ is in $G$, $V_1 \notin X$ and $V_2 \in X \}$. The neighbors of $X$ is the set $ne_G(X) = \{V_1 | V_1 - V_2$ is in $G$, $V_1 \notin X$ and $V_2 \in X \}$. The boundary of $X$ is the set $bd_G(X) = ne_G(X) \cup pa_G(X)$. The adjacents of $X$ is the set $ad_G(X) = ne_G(X) \cup pa_G(X) \cup ch_G(X)$. A route between a node $V_{1}$ and a node $V_{n}$ in $G$ is a sequence of (not necessarily distinct) nodes $V_{1}, \ldots, V_{n}$ st $V_i \in ad_G(V_{i+1})$ for all $1 \leq i < n$. If the nodes in the route are all distinct, then the route is called a path. A route is called undirected if $V_i - V_{i+1}$ is in $G$ for all $1 \leq i < n$. A route is called descending if $V_i \ra V_{i+1}$ or $V_i - V_{i+1}$ is in $G$ for all $1 \leq i < n$. A route is called strictly descending if $V_i \ra V_{i+1}$ is in $G$ for all $1 \leq i < n$. The descendants of a set of nodes $X$ of $G$ is the set $de_G(X) = \{V_n |$ there is a descending path from $V_1$ to $V_n$ in $G$, $V_1 \in X$ and $V_n \notin X \}$. The strict ascendants of $X$ is the set $san_G(X) = \{V_1 |$ there is a strictly descending path from $V_1$ to $V_n$ in $G$, $V_1 \notin X$ and $V_n \in X \}$. A route $V_{1}, \ldots, V_{n}$ in $G$ is called a semidirected cycle if $V_n=V_1$, $V_1 \ra V_2$ is in $G$ and $V_i \ra V_{i+1}$ or $V_i - V_{i+1}$ is in $G$ for all $1 < i < n$. A chain graph (CG) is a graph whose every edge is directed or undirected st it has no semidirected cycles. Note that a CG with only directed edges is a directed and acyclic graph (DAG), and a CG with only undirected edges is an undirected graph (UG). A set of nodes of a CG is connected if there exists an undirected path in the CG between every pair of nodes in the set. A connectivity component of a CG is a maximal connected set. We denote by $co_G(X)$ the connectivity component of the CG $G$ to which a node $X$ belongs. A chain $\alpha$ is a partition of $V$ into ordered subsets, which we call blocks. We say that a CG $G$ and a chain $\alpha$ are consistent when (i) for every edge $X \rightarrow Y$ in $G$, the block containing $X$ precedes the block containing $Y$ in $\alpha$, and (ii) for every edge $X - Y$ in $G$, $X$ and $Y$ are in the same block of $\alpha$. Note that the blocks of $\alpha$ and the connectivity components of $G$ may not coincide, but each of the latter must be included in one of the former.

Let $X$, $Y$, $Z$ and $W$ denote four disjoint subsets of $V$. An independence model $M$ is a set of statements of the form $X \ci_M Y | Z$, meaning that $X$ is independent of $Y$ given $Z$. Moreover, $M$ is called graphoid if it satisfies the following properties: Symmetry $X \ci_M Y | Z \Rightarrow Y \ci_M X | Z$, decomposition $X \ci_M Y \cup W | Z \Rightarrow X \ci_M Y | Z$, weak union $X \ci_M Y \cup W | Z \Rightarrow X \ci_M Y | Z \cup W$, contraction $X \ci_M Y | Z \cup W \land X \ci_M$ $W | Z \Rightarrow X \ci_M Y \cup W | Z$, and intersection $X \ci_M Y | Z \cup W \land X \ci_M W | Z \cup Y \Rightarrow X \ci_M$ $Y \cup W | Z$. Moreover, $M$ is called compositional graphoid if it is a graphoid that also satisfies the composition property $X \ci_M Y | Z \land X \ci_M W | Z \Rightarrow X \ci_M Y \cup W | Z$. By convention, $X \ci_M \emptyset | Z$ and $\emptyset \ci_M Y | Z$.

We now recall the semantics of LWF and AMP CGs. A section of a route $\rho$ in a LWF CG is a maximal undirected subroute of $\rho$. A section $V_{2} - \ldots - V_{n-1}$ of $\rho$ is a collider section of $\rho$ if $V_{1} \rightarrow V_{2} - \ldots - V_{n-1} \leftarrow V_{n}$ is a subroute of $\rho$. Moreover, $\rho$ is said to be $Z$-open with $Z \subseteq V$ when (i) every collider section of $\rho$ has a node in $Z$, and (ii) no non-collider section of $\rho$ has a node in $Z$.

A node $B$ in a route $\rho$ in an AMP CG $G$ is called a triplex node in $\rho$ if $A \ra B \la C$, $A \ra B - C$, or $A - B \la C$ is a subroute of $\rho$. Note that maybe $A=C$ in the first case. Note also that $B$ may be both a triplex and a non-triplex node in $\rho$. Moreover, $\rho$ is said to be $Z$-open with $Z \subseteq V$ when (i) every triplex node in $\rho$ is in $Z$, and (ii) every non-triplex node in $\rho$ is outside $Z$.\footnote{See \cite[Remark 3.1]{Levitzetal.2001} for the equivalence of this and the standard definition of $Z$-open route for AMP CGs.}

Let $X$, $Y$ and $Z$ denote three disjoint subsets of $V$. When there is no $Z$-open route in a LWF or AMP CG $G$ between a node in $X$ and a node in $Y$, we say that $X$ is separated from $Y$ given $Z$ in $G$ and denote it as $X \ci_G Y | Z$. The independence model represented by $G$, denoted as $I(G)$, is the set of separations $X \ci_G Y | Z$. In general, $I(G)$ is different depending on whether $G$ is interpreted as a LWF or AMP CG. However, if $G$ is a DAG or UG, then $I(G)$ is the same under the two interpretations. Given a CG $G$ and two disjoint subsets $L$ and $S$ of $V$, we denote by $[I(G)]_L^S$ the independence model represented by $G$ under marginalization of the nodes in $L$ and conditioning on the nodes in $S$. Specifically, $X \ci_G Y | Z$ is in $[I(G)]_L^S$ iff $X \ci_G Y | Z \cup S$ is in $I(G)$ and $X, Y, Z \subseteq V \setminus L \setminus S$.

We say that a CG $G$ includes an independence model $M$ if $I(G) \subseteq M$. Moreover, we say that $G$ is inclusion minimal wrt $M$ if removing any edge from $G$ makes it cease to include $M$. We say that a CG $G_{\alpha}$ is inclusion minimal wrt an independence model $M$ and a chain $\alpha$ if $G_{\alpha}$ is inclusion minimal wrt $M$ and $G_{\alpha}$ is consistent with $\alpha$. We also say that a CG $G$ is inclusion optimal wrt an independence model $M$ if $I(G) \subseteq M$ and there exists no other CG $H$ st $I(G) \subset I(H) \subseteq M$.

Finally, a subgraph of a CG $G$ is a CG whose nodes and edges are all in $G$. The subgraph of a CG $G$ induced by a set of its nodes $X$ is the CG over $X$ that has all and only the edges in $G$ whose both ends are in $X$. A complex in a LWF CG is an induced subgraph of it of the form $V_{1} \rightarrow V_{2} - \ldots - V_{n-1} \leftarrow V_{n}$. A triplex in an AMP CG is an induced subgraph of it of the form $A \ra B \la C$, $A \ra B - C$, or $A - B \la C$.

\section{Justification of LWF and AMP CGs}\label{sec:justification}

The theorem below shows that every LWF or AMP CG $G$ is inclusion optimal wrt the intersection of the independence models represented by some DAGs under conditioning. The DAGs are obtained as follows. First, we decompose $G$ into a DAG $G_D$ and an UG $G_U$, i.e. $G_D$ contains all and only the directed edges in $G$, and $G_U$ contains all and only the undirected edges in $G$. Then, we construct a DAG $G_S$ from $G_U$ by replacing every edge $X - Y$ in $G_U$ with $X \ra S_{XY} \la Y$. The nodes $S_{XY}$ are called selection nodes. Let $S$ denote all the selection nodes in $G_S$. Note that $G_D$ and $G_U$ are defined over the nodes $V$, but $G_S$ is defined over the nodes $V \cup S$.

\begin{theorem}\label{the:justification}
The LWF or AMP CG $G$ is inclusion optimal wrt $I(G_D) \cap [I(G_S)]_\emptyset^{S}$.
\end{theorem}

\begin{proof}
First, assume that $G$ is a LWF CG. Assume to the contrary that there exists a LWF CG $H$ st $I(G) \subset I(H) \subseteq I(G_D) \cap [I(G_S)]_\emptyset^{S}$. Note that $G$ and $H$ must have the same adjacencies because, otherwise, there are two nodes $X, Y \in V$ that are adjacent in $H$ but not in $G$, or vice versa. The first case implies that $X \ci_G Y | Z$ holds but $X \ci_H Y | Z$ does not hold for some $Z \subseteq V \setminus X \setminus Y$, which contradicts that $I(G) \subset I(H)$. The second case implies that $X \ci_{G_D} Y | Z$ or $X \ci_{G_S} Y | Z \cup S$ does not hold for any $Z \subseteq V \setminus X \setminus Y$. Then, $X \ci Y | Z$ is in $I(H)$ but not in $I(G_D) \cap [I(G_S)]_\emptyset^{S}$, which contradicts that $I(H) \subseteq I(G_D) \cap [I(G_S)]_\emptyset^{S}$. Moreover, if $G$ and $H$ have the same adjacencies, then they must also have the same complexes because, otherwise, there are two nodes $X, Y \in V$ st $X \ci_G Y | Z$ holds but $X \ci_H Y | Z$ does not hold for some $Z \subseteq V \setminus X \setminus Y$, which contradicts that $I(G) \subset I(H)$. However, that $G$ and $H$ have the same adjacencies and complexes contradicts that $I(G) \subset I(H)$ \cite[Theorem 5.6]{Frydenberg1990}.

Now, assume that $G$ is an AMP CG. Assume to the contrary that there exists an AMP CG $H$ st $I(G) \subset I(H) \subseteq I(G_D) \cap [I(G_S)]_\emptyset^{S}$. Note that $G$ and $H$ must have the same adjacencies, by a reasoning similar to the one used above for LWF CGs. Then, they must also have the same triplexes because, otherwise, there are two nodes $X, Y \in V$ st $X \ci_G Y | Z$ holds but $X \ci_H Y | Z$ does not hold for some $Z \subseteq V \setminus X \setminus Y$, which contradicts that $I(G) \subset I(H)$. However, that $G$ and $H$ have the same adjacencies and triplexes contradicts that $I(G) \subset I(H)$ as shown in \cite[Theorem 5]{Anderssonetal.2001} and \cite[Theorem 6.1]{Levitzetal.2001}.
\end{proof}

Unfortunately, the LWF or AMP CG $G$ may not be faithful to $I(G_D) \cap [I(G_S)]_\emptyset^{S}$. To see it, let $G$ be $A \ra B - C \la D$. Then, $A \ci D | B \cup C$ is in $I(G_D) \cap [I(G_S)]_\emptyset^{S}$ but not in $I(G)$. We doubt that one can prove (and so strengthen our justification) that every LWF or AMP CG is faithful to the intersection of the independence models represented by some DAGs under conditioning. However, it is true that the decomposition of $G$ into $G_D$ and $G_U$ is not the only one that allows us to prove that $G$ is inclusion optimal wrt to the intersection of the independence models represented by some DAGs under conditioning. For instance, we can also prove this result if $G$ is decomposed into a set of DAGs and UGs st none of them has more than one edge, or if $G$ is decomposed into a set of CGs st none of them has a subgraph of the form $A \ra B - C$. We omit the proofs. In any case, this does not change the main message of this work, namely that LWF and AMP CGs can be justified on the sole basis of causal models. Having said this, we prefer the original decomposition because it is not completely arbitrary: $G_D$ represents the relationships in $G$ that are causal, and $G_U$ those that are non-causal and need to be explained through conditioning.

Finally, note that the LWF or AMP CG $G$ may not be the only inclusion optimal CG wrt $I(G_D) \cap [I(G_S)]_\emptyset^{S}$. To see it, let $G$ be $A \ra B - C \la D$. Then, any LWF or AMP CG that has the same adjacencies as $G$ is inclusion optimal wrt $I(G_D) \cap [I(G_S)]_\emptyset^{S}$. Some of these other inclusion optimal CGs may even be preferred instead of $G$ according to some criteria (e.g. number of independencies represented, or number of directed and/or undirected edges). However, $G$ is preferred according to an important criterion: It is the only one that has all and only the strictly ascendant relationships (i.e. direct and indirect causal relationships) between two nodes in $V$ that exist in $G_D$ and $G_S$.

\section{Combining Arbitrary DAGs into a LWF or AMP CG}\label{sec:combination}

In this section, we study the opposite of the problem above. Specifically, let $G_1, \ldots, G_r$ denote $r$ arbitrary DAGs, where any $G_i$ is defined over the nodes $V \cup L_i \cup S_i$ and it is subject to marginalization of the nodes in $L_i$ and conditioning on the nodes in $S_i$. We would like to find a LWF or AMP CG that is inclusion optimal wrt $\bigcap_{i=1}^r [I(G_i)]_{L_i}^{S_i}$. However, this seems to be an extremely hard problem. So, we study a simpler version of it in which we are only interested in those CGs that are consistent with a chain $\alpha$. Then, our goal becomes to find an inclusion minimal LWF or AMP CG wrt $\bigcap_{i=1}^r [I(G_i)]_{L_i}^{S_i}$ and $\alpha$. The prior knowledge of $\alpha$ represents our a priori knowledge on which nodes may be causally related and which nodes may be non-causally related. The latter determine the blocks of $\alpha$, and the former the ordering of the blocks in $\alpha$. The theorems below solve our problem. Specifically, they give a constructive characterization of the unique LWF (respectively AMP) CG that is inclusion minimal wrt a graphoid (respectively compositional graphoid) and a chain. Note that any $I(G_i)$ is a compositional graphoid \cite[Theorem 1]{SadeghiandLauritzen2012}. Moreover, it is easy to verify that any $[I(G_i)]_{L_i}^{S_i}$ is also a compositional graphoid and, thus, $\bigcap_{i=1}^r [I(G_i)]_{L_i}^{S_i}$ is also a compositional graphoid. Thus, the theorems below apply to our problem.

\begin{theorem}\label{the:uniquelwf}
Let $M$ denote an independence model, and $\alpha$ a chain with blocks $b_1, \ldots, b_n$. If $M$ is a graphoid, then there exits a unique LWF CG $G_{\alpha}$ that is inclusion minimal wrt $M$ and $\alpha$. Specifically, for each node $X$ of each block $b_i$ of $\alpha$, $bd_{G_{\alpha}}(X)$ is the smallest subset of $\bigcup_{j=1}^i b_j \setminus X$ st $X \ci_M \bigcup_{j=1}^i b_j \setminus X \setminus bd_{G_{\alpha}}(X) | bd_{G_{\alpha}}(X)$.
\end{theorem}

\begin{proof}
The theorem has been proven by \cite[Lemma 1]{Pennaetal.2014}.
\end{proof}

\begin{theorem}\label{the:uniqueamp}
Let $M$ denote an independence model, and $\alpha$ a chain with blocks $b_1, \ldots, b_n$. If $M$ is a compositional graphoid, then there exits a unique AMP CG $G_{\alpha}$ that is inclusion minimal wrt $M$ and $\alpha$. Specifically, consider the blocks in $\alpha$ in reverse order and perform the following two steps for each of them. First, for each node $X$ of the block $b_i$, $ne_{G_{\alpha}}(X)$ is the smallest subset of $b_i \setminus X$ st $X \ci_M b_i \setminus X \setminus ne_{G_{\alpha}}(X) | \bigcup_{j=1}^{i-1} b_j \cup ne_{G_{\alpha}}(X)$. Second, for each node $X$ of the block $b_i$, $pa_{G_{\alpha}}(X)$ is the smallest subset of $\bigcup_{j=1}^{i-1} b_j$ st $X \ci_M V \setminus X \setminus de_{G_{\alpha}}(X) \setminus pa_{G_{\alpha}}(X) | pa_{G_{\alpha}}(X)$.\footnote{Note that $de_{G_{\alpha}}(X)$ for any $X \in b_i$ is known when the second step for $b_i$ starts, because $ne_{G_{\alpha}}(X)$ for any $X \in \bigcup_{j=i}^{n} b_j$ and $pa_{G_{\alpha}}(X)$ for any $X \in \bigcup_{j=i+1}^{n} b_j$ have already been identified.}
\end{theorem}

\begin{proof}
Consider any $X \in V$. Assume that $X \in b_i$. By construction, we have that

$Y \ci_M V \setminus Y \setminus de_{G_{\alpha}}(Y) \setminus pa_{G_{\alpha}}(Y) | pa_{G_{\alpha}}(Y)$ 

for any $Y \in X \cup ne_{G_{\alpha}}(X)$. Then,

$Y \ci_M V \setminus Y \setminus de_{G_{\alpha}}(Y) \setminus pa_{G_{\alpha}}(X \cup ne_{G_{\alpha}}(X)) | pa_{G_{\alpha}}(X \cup ne_{G_{\alpha}}(X))$ 

for any $Y \in X \cup ne_{G_{\alpha}}(X)$ by weak union. Then,

$X \cup ne_{G_{\alpha}}(X) \ci_M V \setminus X \setminus de_{G_{\alpha}}(X) \setminus pa_{G_{\alpha}}(X \cup ne_{G_{\alpha}}(X)) | pa_{G_{\alpha}}(X \cup ne_{G_{\alpha}}(X))$ 

by repeated application of symmetry and composition. Then,

$X \ci_M V \setminus X \setminus de_{G_{\alpha}}(X) \setminus pa_{G_{\alpha}}(X \cup ne_{G_{\alpha}}(X)) | pa_{G_{\alpha}}(X \cup ne_{G_{\alpha}}(X)) \cup ne_{G_{\alpha}}(X)$ 

by symmetry and weak union. Then,

$X \ci_M \bigcup_{j=1}^{i-1} b_j \setminus pa_{G_{\alpha}}(X \cup ne_{G_{\alpha}}(X)) | pa_{G_{\alpha}}(X \cup ne_{G_{\alpha}}(X)) \cup ne_{G_{\alpha}}(X)$

by decomposition. This together with

$X \ci_M b_i \setminus X \setminus ne_{G_{\alpha}}(X) | \bigcup_{j=1}^{i-1} b_j \cup ne_{G_{\alpha}}(X)$

which follows by construction, imply that

$X \ci_M co_{G_{\alpha}}(X) \setminus X \setminus ne_{G_{\alpha}}(X) | pa_{G_{\alpha}}(X \cup ne_{G_{\alpha}}(X)) \cup ne_{G_{\alpha}}(X)$

by contraction and decomposition. This together with

$X \ci_M V \setminus X \setminus de_{G_{\alpha}}(X) \setminus pa_{G_{\alpha}}(X) | pa_{G_{\alpha}}(X)$

which follows by construction, imply by decomposition that

$X \ci_M Y | pa_{G_{\alpha}}(X \cup ne_{G_{\alpha}}(X)) \cup ne_{G_{\alpha}}(X)$

for any $Y \in co_{G_{\alpha}}(X) \setminus X \setminus ne_{G_{\alpha}}(X)$, and

$X \ci_M Y | pa_{G_{\alpha}}(X)$

for any $Y \in V \setminus X \setminus de_{G_{\alpha}}(X) \setminus pa_{G_{\alpha}}(X)$. These independencies plus those that can be derived from them by applying the compositional graphoid properties are exactly the independencies in $I(G_{\alpha})$ \cite[Theorems 5 and 6]{Penna2014}.\footnote{Theorems 5 and 6 in the work of \cite{Penna2014} are stated for so-called marginal AMP CGs. However, they also apply to AMP CGs because these are marginal AMP CGs without bidirected edges.} This implies that $G_{\alpha}$ includes $M$.\footnote{This result may also be derived by adapting to general independence models the results reported by \cite[Section 4]{Anderssonetal.2001} for probability distributions.} In fact, $G_{\alpha}$ is inclusion minimal wrt $M$ and $\alpha$ by construction of $ne_{G_{\alpha}}(X)$ and $pa_{G_{\alpha}}(X)$.

Assume to the contrary that there exists another AMP CG $H_{\alpha}$ that is inclusion minimal wrt $M$ and $\alpha$. Let $X \in V$ denote any node st $ne_{G_{\alpha}}(X) \neq ne_{H_{\alpha}}(X)$. Assume that $X \in b_i$. Then, 

$X \ci_M b_i \setminus X \setminus ne_{G_{\alpha}}(X) | \bigcup_{j=1}^{i-1} b_j \cup ne_{G_{\alpha}}(X)$

and

$X \ci_M b_i \setminus X \setminus ne_{H_{\alpha}}(X) | \bigcup_{j=1}^{i-1} b_j \cup ne_{H_{\alpha}}(X)$

because $G_{\alpha}$ and $H_{\alpha}$ include $M$. Then, 

$X \ci_M b_i \setminus X \setminus [ ne_{G_{\alpha}}(X) \cap ne_{H_{\alpha}}(X) ] | \bigcup_{j=1}^{i-1} b_j \cup [ ne_{G_{\alpha}}(X) \cap ne_{H_{\alpha}}(X) ]$

by intersection. However, this contradicts the definition of $ne_{G_{\alpha}}(X)$, because $ne_{G_{\alpha}}(X) \cap ne_{H_{\alpha}}(X)$ is smaller than $ne_{G_{\alpha}}(X)$. Consequently, $ne_{G_{\alpha}}(X) = ne_{H_{\alpha}}(X)$ for any $X \in V$.

Let $i$ denote the largest block index st there is some $X \in b_i$ st $pa_{G_{\alpha}}(X) \neq pa_{H_{\alpha}}(X)$. Note that $de_{G_{\alpha}}(X)=de_{H_{\alpha}}(X)$, because $pa_{G_{\alpha}}(Y) = pa_{H_{\alpha}}(Y)$ for any $Y \in \bigcup_{j=i+1}^{n} b_j$ and, as proven above, $ne_{G_{\alpha}}(Y) = ne_{H_{\alpha}}(Y)$ for any $Y \in V$. Then, 

$X \ci_M V \setminus X \setminus de_{G_{\alpha}}(X) \setminus pa_{G_{\alpha}}(X) | pa_{G_{\alpha}}(X)$ 

and 

$X \ci_M V \setminus X \setminus de_{H_{\alpha}}(X) \setminus pa_{H_{\alpha}}(X) | pa_{H_{\alpha}}(X)$ 

because $G_{\alpha}$ and $H_{\alpha}$ include $M$. Then, 

$X \ci_M V \setminus X \setminus de_{G_{\alpha}}(X) \setminus [ pa_{G_{\alpha}}(X) \cap pa_{H_{\alpha}}(X) ] | [ pa_{G_{\alpha}}(X) \cap pa_{H_{\alpha}}(X) ]$ 

by intersection. However, this contradicts the definition of $pa_{G_{\alpha}}(X)$, because $pa_{G_{\alpha}}(X) \cap pa_{H_{\alpha}}(X)$ is smaller than $pa_{G_{\alpha}}(X)$. Consequently, $pa_{G_{\alpha}}(X) = pa_{H_{\alpha}}(X)$ for any $X \in V$. Therefore, $G_{\alpha}$ and $H_{\alpha}$ have the same edges, which is a contradiction.
\end{proof}

\section{Discussion}\label{sec:discussion}

The purpose of this paper has been to justify LWF and AMP CGs by showing that they do not represent arbitrary independence models. Unlike previous justifications, ours builds solely on causal models and does not involve equilibrium distributions or deterministic nodes, which may seem odd to some readers. Specifically, for any given LWF or AMP CG, we have imagined a system that switches between different regimes or configurations, and we have shown that the given CG represents the different regimes jointly. To do so, we have assumed that each of the regimes can be represented by a causal model. We have also assumed that the causal models may be subject to selection bias. In other words, we have assumed that each of the regimes can be represented by a DAG under conditioning.

In this paper, we have also studied the opposite of the problem above, namely how to combine a set of arbitrary DAGs under marginalization and conditioning into a consensus LWF or AMP CG. We have shown how to do it optimally when the consensus CG must be consistent with a given chain. The chain may represent our prior knowledge about the causal and non-causal relationships in the domain at hand. In the future, we would like to drop this requirement. We would also like to find parameter values for the consensus CG by combining the parameter values associated to the given DAGs.

\subsubsection*{Acknowledgments.}

This work is funded by the Center for Industrial Information Technology (CENIIT) and a so-called career contract at Link\"oping University, and by the Swedish Research Council (ref. 2010-4808).

\end{document}